\documentclass{article}

\usepackage{microtype}
\usepackage{graphicx}
\usepackage{subcaption}
\usepackage{booktabs} 

\usepackage{hyperref}

\usepackage[preprint]{icml2026}

\usepackage{amsmath}
\usepackage{amssymb}
\usepackage{mathtools}
\usepackage{amsthm}

\usepackage{subcaption}
\usepackage{tikz-cd}
\usetikzlibrary{arrows}
\usepackage{comment}
\usepackage{amssymb}
\usepackage{tikz}
\usepackage{tikz-qtree,tikz-qtree-compat}
\usetikzlibrary{positioning}

\usepackage[capitalize,noabbrev]{cleveref}

\theoremstyle{plain}
\newtheorem{theorem}{Theorem}[section]

\newtheorem{corollary}[theorem]{Corollary}
\theoremstyle{definition}
\newtheorem{definition}[theorem]{Definition}

\theoremstyle{remark}

\newcommand\independent{\protect\mathpalette{\protect\independenT}{\perp}}
\def\independenT#1#2{\mathrel{\rlap{$#1#2$}\mkern2mu{#1#2}}}

\usepackage{stackengine}
\def\delequal{\mathrel{\ensurestackMath{\stackon[1pt]{=}{\scriptstyle\Delta}}}}

\usepackage[textsize=tiny]{todonotes}

\icmltitlerunning{Bounding Probabilities of Causation with Partial Causal Diagrams}

\begin{document}

\twocolumn[
  \icmltitle{Bounding Probabilities of Causation with Partial Causal Diagrams}



  \icmlsetsymbol{equal}{*}

  \begin{icmlauthorlist}
    \icmlauthor{Yuxuan Xie}{yyy}
    \icmlauthor{Ang Li}{sch}
  \end{icmlauthorlist}

  \icmlaffiliation{yyy}{University of Ottawa, Ontario, Canada}
  \icmlaffiliation{sch}{Department of Computer Science, Florida State University, FL, U.S.A.}

  \icmlcorrespondingauthor{Yuxuan Xie}{edithxie944@gmail.com}
  \icmlcorrespondingauthor{Ang Li}{angli@cs.fsu.edu}

  \icmlkeywords{Machine Learning, ICML}

  \vskip 0.3in
]



\printAffiliationsAndNotice{}  

\begin{abstract}
      Probabilities of causation are fundamental to individual-level explanation and decision making, yet they are inherently counterfactual and not point-identifiable from data in general. Existing bounds either disregard available covariates, require complete causal graphs, or rely on restrictive binary settings, limiting their practical use. In real-world applications, causal information is often partial but nontrivial. This paper proposes a general framework for bounding probabilities of causation using partial causal information. We show how the available structural or statistical information can be systematically incorporated as constraints in a optimization programming formulation, yielding tighter and formally valid bounds without full identifiability. This approach extends the applicability of probabilities of causation to realistic settings where causal knowledge is incomplete but informative.
\end{abstract}

\section{Introduction}
Many important decisions in medicine, policy, and law require reasoning about causes of observed effects at the individual level \cite{lu2011budgeted,li2019unit,mueller2023personalized}. Typical questions include whether a treatment caused a patient's recovery, whether an intervention was necessary for an outcome, or whether a policy decision was sufficient to prevent harm. Such questions concern counterfactual relationships and cannot be answered using population level causal effects alone, even when such effects are identified from randomized experiments.

Probabilities of causation (PoCs) \cite{pearl1999probabilities,galles1998axiomatic,halpern2000axiomatizing,tian2000probabilities} with Structural Causal Models \cite{pearl2009causality} provide a formal framework for addressing these individualized causal questions. Well studied examples include the probability of necessity (PN), the probability of sufficiency (PS), and the probability of necessity and sufficiency (PNS) \cite{pearl1999probabilities}. These quantities are central to applications such as treatment prioritization, legal attribution of responsibility, and personalized decision making. However, probabilities of causation are generally not point identifiable from data, because they depend on joint counterfactual events that cannot be simultaneously observed. As a result, research in this area has focused on deriving informative bounds rather than point estimates \cite{tian2000probabilities,pearl:etal21-r505,li2024probabilities}.

Classical results show that sharp bounds on probabilities of causation can be obtained by combining experimental and observational data while making no assumptions about the underlying data generating process \cite{tian2000probabilities}. While these bounds are valid under minimal assumptions, they are often wide and may be insufficient for practical decision making. Subsequent work demonstrated that additional assumptions in the form of a fully specified causal diagram can substantially narrow these bounds \cite{pearl:etal21-r505,zhang2022partial}. Even when experimental and observational distributions are fixed, additional structural assumptions can restrict the set of counterfactual distributions compatible with the data, leading to tighter bounds on probabilities of causation.

Despite their theoretical importance, existing graph based approaches have limited practical applicability. They typically require complete knowledge of the causal structure, including all relevant variables and their relationships \cite{zhang2022partial}, and most available results are restricted to binary treatments and binary outcomes \cite{pearl:etal21-r505}. In many real world settings, such assumptions are unrealistic. Domain knowledge is often partial, providing only incomplete causal information, such as the presence or absence of certain causal paths, qualitative ordering constraints, or limited knowledge about confounding mechanisms.

In this paper, we develop a general framework for incorporating partial causal information into the analysis of probabilities of causation. Similar to Balke's linear programming \cite{balke1995probabilistic,tian2000probabilities}, our approach formulates probabilities of causation as solutions to optimization problems over counterfactual distributions and allows partial causal knowledge to be systematically encoded as constraints in these programs. These constraints may arise from incomplete causal diagrams, background scientific knowledge, or structural assumptions that are weaker than full causal diagram. By integrating such information directly into the optimization formulation, we obtain tighter bounds while maintaining validity under uncertainty about the remaining causal structure.

The proposed framework generalizes existing results in several important directions. It subsumes bounds derived from fully specified causal graphs as special cases, extends naturally beyond binary treatments and outcomes, and provides a unified mechanism for combining experimental data, observational data, and partial causal assumptions. By making explicit how partial causal information restricts the feasible set of counterfactual distributions, our approach offers a principled and flexible tool for individualized causal reasoning in realistic settings.

\paragraph{Contributions.}
We propose an optimization-based framework for bounding probabilities of causation under partial causal knowledge, where structural and auxiliary statistical information are encoded as constraints on counterfactual distributions. The framework (i) generalizes prior bounds beyond binary treatments and outcomes, (ii) incorporates covariate information modularly without requiring joint covariate distributions, (iii) combines back-door covariates and mediator information in a unified program to further tighten bounds, and (iv) is validated empirically via simulations showing consistent improvements over Tian--Pearl and Mueller--Li--Pearl baselines.

\section{Preliminaries and Related Work}
\label{related work}

In this section, we review standard definitions of probabilities of causation following \cite{pearl1999probabilities}. Our analysis is formulated using counterfactual events interpreted under the Structural Causal Model (SCM) semantics \cite{galles1998axiomatic,halpern2000axiomatizing,balke2013counterfactuals}. We use causal diagrams as a representational tool when structural assumptions are available, but the definitions themselves do not depend on a fully specified causal graph.

Let $Y_x = y$ denote the counterfactual statement ``$Y$ would take value $y$ if $X$ were set to $x$.'' For notational convenience, throughout the paper we write $y_x$ for the event $Y_x = y$, $y_{x'}$ for $Y_{x'} = y$, $y'_x$ for $Y_x = y'$, and $y'_{x'}$ for $Y_{x'} = y'$. However, the proposed framework is not restricted to the binary case.

We consider the following three probabilities of causation.

\begin{definition}[Probability of necessity (PN)]
Let $X$ and $Y$ be two binary variables in a causal model $M$, let $x$ and $y$ stand for the propositions $X=\text{true}$ and $Y=\text{true}$, respectively, and $x'$ and $y'$ for their complements. The probability of necessity is defined as the expression \cite{pearl1999probabilities}\\
\begin{eqnarray}
\text{PN} & \delequal & P(Y_{x'}=\text{false}|X=\text{true},Y=\text{true})\nonumber \\ 
& \delequal & P(y'_{x'}|x,y)
\label{pn}
\end{eqnarray}
\end{definition}

\begin{definition}[Probability of sufficiency (PS)] \cite{pearl1999probabilities}
\begin{eqnarray}
\text{PS}\delequal P(y_x|y',x')
\label{ps}
\end{eqnarray}
\end{definition}

\begin{definition}[Probability of necessity and sufficiency (PNS)] \cite{pearl1999probabilities}
\begin{eqnarray}
\text{PNS}\delequal P(y_x,y'_{x'})
\label{pns}
\end{eqnarray}
\end{definition}

PNS represents the probability that the outcome would occur under treatment and would not occur in its absence, and therefore captures both the necessity and sufficiency of $X$ for producing $Y$.

Tian and Pearl \cite{tian2000probabilities} derived tight bounds for PNS, PN, and PS using an optimization-based formulation that places no assumptions on the underlying causal structure beyond consistency with experimental and observational data. These bounds are obtained by solving a linear program over admissible counterfactual distributions, often referred to as Balke's linear programming \cite{balke1997probabilistic}. Li and Pearl \cite{li2019unit} later provided a theoretical proof of these bounds.

The tight bounds for PNS and PN are given by

\begin{eqnarray}
\text{PNS} \ge \max \left \{
\begin{array}{cc}
0 \\
P(y_x) - P(y_{x'}) \\
P(y) - P(y_{x'}) \\
P(y_x) - P(y)\\
\end{array}
\right \}
\label{pnslb}
\end{eqnarray}

\begin{eqnarray}
\text{PNS} \le \min \left \{
\begin{array}{cc}
 P(y_x) \\
 P(y'_{x'}) \\
P(x,y) + P(x',y') \\
P(y_x) - P(y_{x'}) +\\
+ P(x, y') + P(x', y)
\end{array} 
\right \}
\label{pnsub}
\end{eqnarray}

\begin{eqnarray}
\text{PN} \ge \max \left \{
\begin{array}{cc}
0 \\
\frac{P(y)-P(y_{x'})}{P(x,y)}
\end{array} 
\right \}
\label{pnlb}
\end{eqnarray}

\begin{eqnarray}
\text{PN} \le
\min \left \{
\begin{array}{cc}
1 \\
\frac{P(y'_{x'})-P(x',y')}{P(x,y)} 
\end{array}
\right \}
\label{pnub}
\end{eqnarray}

Bounds for PS follow directly by symmetry, exchanging $x$ with $x'$ and $y$ with $y'$.

For subpopulations defined by a set of observed characteristics $C$, all expressions above can be conditioned on $C=c$. In this work, we focus on a different source of information. Rather than conditioning on observed covariates alone or assuming a fully specified causal diagram, we study how partial causal information can be incorporated as constraints into the underlying optimization problem. This perspective allows probabilities of causation to be bounded more tightly by restricting the feasible set of counterfactual distributions in a principled and modular way. Extensions to PN and PS follow analogously from the same construction.

\section{Bounds with Partial Causal Information}
\subsection{Non-descendant Covariates}
Theorems 4 and 5 in \cite{pearl:etal21-r505} derived bounds on PNS under the assumption that a set of observed variables, $Z$, is available and that $Z$ contains no descendants of the treatment variable $X$. This assumption guarantees that counterfactual quantities such as $P(Y_X|Z)$ are well defined from the available data and that conditioning on $Z$ can impose additional constraints on the joint counterfactual distribution. Under these conditions, the resulting bounds were shown to be contained within the Tian-Pearl bounds. However, the treatment and outcome variables $X$ and $Y$ are restricted to the binary case, and the analysis assumes access to covariate-specific information for all values of $Z$ ($Z$ is a set of covariates), namely $P(Z)$ and $P(Y_x | z)$ for every $z \in Z$.

In the present work, we remove this requirement. We allow treatments and outcomes to take arbitrary finite values and permit partial covariate information to be incorporated in a modular manner. In particular, covariate-specific distributions need not be jointly available. For example, when multiple confounders \(Z_1\) and \(Z_2\) are present, the available data may consist of \(P(Y_x \mid Z_1)\) and \(P(Z_1)\) separately from \(P(Y_x \mid Z_2)\) and \(P(Z_2)\). Throughout, we adopt the definition of multivalued PNS introduced in Shu, Wang, and Li \cite{shu2025identification}, which provides a complete and unified characterization for arbitrary discrete treatments and outcomes.

In subsequent theorems, we further show how such covariate information can be combined with information on mediating variables. Confounders and mediators are treated within the same optimization-based formulation and can be incorporated jointly to impose additional constraints on the admissible counterfactual distributions. Rather than relying on variable-specific measurability arguments, we express the available data and structural assumptions as constraints in an optimization problem over counterfactual distributions. This formulation provides a unified treatment of different types of auxiliary information.

The main theorem in this paper gives sufficient conditions under which such constraints are valid and yields corresponding bounds on probabilities of causation. Mueller, Li, and Pearl's results \cite{pearl:etal21-r505} on non-descendant covariates and binary variables follow as a special case of this general formulation.

\begin{theorem}
Given a partial causal diagram $G$ and distribution compatible with $G$, let $Z_1,...,Z_m$ be a set of $m$ variables that does not contain any descendant of $X$ in $G$. Let the variable $X$ has $n$ values $x_1,...,x_n$ and $Y$ has $n$ values $y_1,...,y_n$, $k \le n$, then the probability of necessity and sufficiency(k) $P({y_{1}}_{x_{1}},...,{y_{k}}_{x_{k}})$ (PNS($k$)) is bounded as following:

\begin{eqnarray*}
\text{LB}\le PNS(k) \le \text{UB}
\end{eqnarray*}
where LB and UB is the min and max solution to the following linear optimization problem, where the variables $p_{j_1...j_{n+m+1}}$ represent the probability $P(Y_{x_1},...,Y_{x_n},Z_1,...,Z_m,X)$, where $j_1,...,j_n,j_{n+m+1}\in [1,n]$ and $j_{n+i}\in [1,|Z_i|]$ for $i\in[1,m]$.

\begin{equation*}
    max / min {\sum_{{j_{k+1},...,j_{n+m+1}}} {p_{1 ... k {j_{k+1}}... j_{n+m+1}}}}
\end{equation*}
and along with linear constraints:
\begin{eqnarray*}
    \sum_{j_1,...,j_{n+m+1}}{p_{j_1 ... j_{n+m+1}}} &=& 1,\\
    p_{j_1 ... j_{n+m+1}} &\ge& 0,
\end{eqnarray*}
$\forall s,t \in [1,n], r_i\in[1,|Z_i|],Z\subseteq \{Z_1,...,Z_m\}$:
\begin{eqnarray}
    \sum_{\substack{j_1,...,j_{t-1},\\j_{t+1},...,j_n,\\j_{n+m+1},\\j_{n+i}\text{ for }Z_i\notin Z}}p_{\substack{j_1 ... j_{t-1} s j_{t+1}... j_{n} [j_{n+i} \text{ if } Z_i\notin Z,r_i\text{ o.w.}]j_{n+m+1}}}\nonumber
    \\=P({y_s}_{x_t},z_{ir_i}\text{ for }Z_i\in Z),\label{thm4eq1}
\end{eqnarray}
\begin{eqnarray}
    \sum_{\substack{j_1,...j_{t-1},\\j_{t+1},...,j_{n},\\j_{n+i}\text{ for }Z_i\notin Z}}{p_{\substack{j_1 ... j_{t-1} s j_{t+1}... j_{n} [j_{n+i} \text{ if } Z_i\notin Z,r_i\text{ o.w.}] t}}}
    \nonumber\\= P({x_t},{y_s},z_{ir_i}\text{ for }Z_i\in Z).\label{thm4eq2}
\end{eqnarray}
\label{tm1}
\end{theorem}

The special case of the above Theorem \ref{tm1} with $n = k = 2$ coincides with Theorem 4 in \cite{pearl:etal21-r505}. The constraints in Equations~\ref{thm4eq1} and~\ref{thm4eq2} enumerate all possible combinations when covariates $Z_1,\ldots,Z_m$ are jointly available, thereby yielding the tightest bounds. However, in general, the applicable constraints depend on the availability of data. For example, in many real-world settings, one may only observe covariate-specific data, such as $P(X,Y,Z_i)$ and $P(Y_X,Z_i)$ for each $i \in [1,n]$, without access to the full joint distributions $P(X,Y,Z_1,\ldots,Z_n)$ or $P(Y_X,Z_1,\ldots,Z_n)$. Even in this case, the proposed Theorem \ref{tm1} still integrates information across covariates, rather than treating them separately, and thus produces tighter bounds, as formalized in the following corollary.

\begin{corollary}
Given a partial causal diagram $G$ and distribution compatible with $G$, let $Z_1,...,Z_m$ be a set of $m$ variables that does not contain any descendant of $X$ in $G$. Let the variable $X$ has $n$ values $x_1,...,x_n$ and $Y$ has $n$ values $y_1,...,y_n$, $k \le n$, then the probability of necessity and sufficiency(k) $P({y_{1}}_{x_{1}},...,{y_{k}}_{x_{k}})$ (PNS($k$)) is bounded as following:

\begin{eqnarray*}
\text{LB}\le PNS(k) \le \text{UB}
\end{eqnarray*}
where LB and UB is the min and max solution to the following linear optimization problem, where the variables $p_{j_1...j_{n+m+1}}$ represent the probability $P(Y_{x_1},...,Y_{x_n},Z_1,...,Z_m,X)$, where $j_1,...,j_n,j_{n+m+1}\in [1,n]$ and $j_{n+i}\in [1,|Z_i|]$ for $i\in[1,m]$.

\begin{equation*}
    max / min {\sum_{{j_{k+1},...,j_{n+m+1}}} {p_{1 ... k {j_{k+1}}... j_{n+m+1}}}}
\end{equation*}
and along with linear constraints:
\begin{eqnarray*}
    \sum_{j_1,...,j_{n+m+1}}{p_{j_1 ... j_{n+m+1}}} &=& 1,\\
    p_{j_1 ... j_{n+m+1}} &\ge& 0,
\end{eqnarray*}
$\forall s,t \in [1,n], i\in[1,..,m], r_i\in[1,|Z_i|]$:
\begin{eqnarray}
    \sum_{\substack{j_1,...j_{t-1},\\j_{t+1},...,j_{n+i-1},\\j_{n+i+1}...,j_{n+m+1}}}{p_{j_1 ... j_{t-1} s j_{t+1}... j_{n+i-1}r_ij_{n+i+1}...j_{n+m+1}}} 
    \nonumber\\=P({y_s}_{x_t},z_{ir_i}),\label{tm5eq1}
\end{eqnarray}
\begin{eqnarray}
        \sum_{\substack{j_1,...j_{t-1},\\j_{t+1},...,j_{n+i-1},\\j_{n+i+1}...,j_{n+m}}}{p_{j_1 ... j_{t-1} s j_{t+1}... j_{n+i-1}r_ij_{n+i+1}...j_{n+m}t}} \nonumber
        \\=P({x_t},{y_s},z_{ir_i}).\label{tm5eq2}
\end{eqnarray}

\label{tm2}
\end{corollary}

The above Theorem~\ref{tm2} represents a special case in which only covariate-specific data are available. Indeed, if joint distributions over any subset of covariates are also available (e.g., $P(X,Y,Z_1,Z_2)$), such information can be incorporated into Theorem~\ref{tm1} to obtain tighter bounds.

\subsection{Backdoor Set with Mediators}
The back-door criterion \cite{pearl1993aspects} is a special case of non-descendant covariates, to which Theorem~\ref{tm1} can be applied to obtain bounds on $PNS(k)$. However, incorporating additional information from mediators can further tighten these bounds. Theorem~6 in \cite{pearl:etal21-r505} provides a binary illustration of this idea in a toy example with an empty back-door set. The following theorem systematically characterizes bounds on $PNS(k)$ in the presence of such additional information.

Given a causal diagram $G$ and distribution compatible with $G$, let $Z$ be a set of variables satisfying the back-door criterion \cite{pearl1993aspects} in $G$, then the PNS is bounded as follows:

\begin{theorem}
Given a causal diagram $G$ and distribution compatible with $G$, let $Z_1,...,Z_m$ be a set of $m$ variables satisfying the back-door criterion \cite{pearl1993aspects} in $G$, and $W$ is a mediator of $X$ and $Y$. Let the variable $X$ has $n$ values $x_1,...,x_n$ and $Y$ has $n$ values $y_1,...,y_n$, $k \le n$, then the probability of necessity and sufficiency(k) $P({y_{1}}_{x_{1}},...,{y_{k}}_{x_{k}})$ (PNS($k$)) is bounded as following:

\begin{eqnarray*}
\text{LB}\le PNS(k) \le \text{UB}
\end{eqnarray*}
where LB and UB is the min and max solution to the following optimization problem, where the variables $p_{j_1...j_{2n+m+2}}$ represent the probability $P(Y_{x_1},...,Y_{x_n},W_{x_1},...,W_{x_n},Z_1,...,Z_m,W,X)$, where $j_1,...,j_n,j_{2n+m+2}\in [1,n]$ and $j_{2n+i}\in [1,|Z_i|]$ for $i\in[1,m]$, and $j_{2n+m+1},j_{n+j}\in [1,|W|]$ for $j\in[1,n]$.

\begin{equation*}
    max / min {\sum_{{j_{k+1},...,j_{2n+m+2}}} {p_{1 ... k {j_{k+1}}... j_{2n+m+2}}}}
\end{equation*}
and along with linear constraints:
\begin{eqnarray*}
    \sum_{j_1,...,j_{2n+m+2}}{p_{j_1 ... j_{2n+m+2}}} &=& 1,\\
    p_{j_1 ... j_{2n+m+2}} &\ge& 0,
\end{eqnarray*}
$\forall s,t \in [1,n], r_i\in[1,|Z_i|],Z\subseteq \{Z_1,...,Z_m\},u\in[1,|W|]$:
\begin{eqnarray}
    \sum_{\substack{j_1,...j_{t-1},\\j_{t+1},...j_{2n},\\j_{2n+m+1},j_{2n+m+2},\\j_{2n+i}\text{ for }Z_i\notin Z}}{p_{\substack{j_1 ... j_{t-1} s j_{t+1}... j_{2n}\\ [j_{2n+i} \text{ if } Z_i\notin Z,r_i\text{ o.w.}]j_{2n+m+1}j_{2n+m+2}}}} \nonumber\\=P({y_s}_{x_t},z_{ir_i}\text{ for }Z_i\in Z),\label{thm6eq1}
\end{eqnarray}
\begin{eqnarray}
    \sum_{\substack{j_1,...j_{t-1},\\j_{t+1},...,j_{2n},\\j_{2n+i}\text{ for }Z_i\notin Z}}{p_{j_1 ... j_{t-1} s j_{t+1}... j_{2n} [j_{2n+i} \text{ if } Z_i\notin Z,r_i\text{ o.w.}] ut}} \nonumber\\= P({x_t},{y_s},w_u,z_{ir_i}\text{ for }Z_i\in Z),\label{thm6eq2}
\end{eqnarray}

$\forall s,t,v \in [1,n], u\in[1,|W|],z_i\in[1,|Z_i|] \text{ for } i\in[1,m]$:\\
let
\begin{eqnarray}
    a&=&\sum_{\substack{j_1,...j_{t-1},\\j_{t+1},...,j_{n+t-1},\\j_{n+t+1},...,j_{2n},\\j_{2n+m+1},j_{2n+m+2}}}{p_{\substack{j_1 ... j_{t-1} s j_{t+1}... j_{n+t-1}u\\j_{n+t+1}...j_{2n}z_1...z_m\\j_{2n+m+1}j_{2n+m+2}}}} ,\nonumber\\
    b&=&\sum_{\substack{j_1,...j_{t-1},\\j_{t+1},...,j_{n+t-1},\\j_{n+t+1},...,j_{2n},\\j_{2n+m+1}}}{p_{\substack{j_1 ... j_{t-1} s j_{t+1}...j_{n+t-1}u\\j_{n+t+1}...j_{2n}z_1...z_mj_{2n+m+1}v}}},\nonumber\\
    c&=&\sum_{\substack{j_1,...,j_{n+t-1},\\j_{n+t+1},...,j_{2n},\\j_{2n+m+1},j_{2n+m+2}}}{p_{\substack{j_1...j_{n+t-1}uj_{n+t+1}...j_{2n}\\z_1...z_mj_{2n+m+1}j_{2n+m+2}}}},\nonumber\\
    d&=&\sum_{\substack{j_1,...,j_{n+t-1},\\j_{n+t+1},...,j_{2n},\\j_{2n+m+1}}}{p_{\substack{j_1 ...j_{n+t-1}uj_{n+t+1}...j_{2n}\\z_1...z_mj_{2n+m+1}v}}},\nonumber
\end{eqnarray}
\begin{eqnarray}
    a\times c = b\times d,\label{thm6eq3}
\end{eqnarray}

$\forall s,t,t'\in [1,n], t<t', u,v\in[1,|W|],z_i\in[1,|Z_i|] \text{ for } i\in[1,m]$:\\
let
\begin{eqnarray}
    e&=&\sum_{\substack{j_1,...j_{t-1},\\j_{t+1},...,j_{n+t-1},\\j_{n+t+1},...,j_{2n},\\j_{2n+m+1},j_{2n+m+2}}}{p_{\substack{j_1 ... j_{t-1} s j_{t+1}... j_{n+t-1}u\\j_{n+t+1}...j_{2n}z_1...z_m\\j_{2n+m+1}j_{2n+m+2}}}},\nonumber\\
    f&=&\sum_{\substack{j_1,...j_{t-1},\\j_{t+1},...,j_{n+t-1},\\j_{n+t+1},...,j_{n+t'-1},\\j_{n+t'+1},...,j_{2n},\\j_{2n+m+1},j_{2n+m+2}}}{p_{\substack{j_1 ... j_{t-1} s j_{t+1}... j_{n+t-1}u\\j_{n+t+1}...j_{n+t'-1}v\\j_{n+t'+1}...j_{2n}z_1...z_m\\j_{2n+m+1}j_{2n+m+2}}}},\nonumber\\
    g&=&\sum_{\substack{j_1,...,j_{n+t-1},\\j_{n+t+1},...,j_{2n},\\j_{2n+m+1},j_{2n+m+2}}}{p_{\substack{j_1...j_{n+t-1}u\\j_{n+t+1}...j_{2n}z_1...z_m\\j_{2n+m+1}j_{2n+m+2}}}},\nonumber\\
    h&=&\sum_{\substack{j_1,...,j_{n+t-1},\\j_{n+t+1},...,j_{n+t'-1},\\j_{n+t'+1},...,j_{2n},\\j_{2n+m+1},j_{2n+m+2}}}{p_{\substack{j_1 ...j_{n+t-1}u\\j_{n+t+1}...j_{n+t'-1}v\\j_{n+t'+1}...j_{2n}z_1...z_m\\j_{2n+m+1}j_{2n+m+2}}}},\nonumber
\end{eqnarray}
\begin{eqnarray}
    e\times g = f\times h.\label{thm6eq4}
\end{eqnarray}
\label{tm3}
\end{theorem}

\noindent\textbf{Remark.}
Note that, similar to Theorem~\ref{tm1}, Equations~\ref{thm6eq1} and~\ref{thm6eq2} are optional and depend on data availability. Equations~\ref{thm6eq3} enforce the conditional independence $Y_x \independent X \mid W_x,Z_1,...,Z_m$ for all $x \in X$, while Equations~\ref{thm6eq4} enforce $Y_x \independent W_{x'} \mid W_x,Z_1,...,Z_m$ for all $x,x' \in X$ with $x \neq x'$. The detailed proofs of all the theorems are provided in the appendix. Moreover, the causal graph $G$ may be a partial causal graph, provided that $Z_1,\ldots,Z_m$ form a valid back-door set. Finally, $W$ may represent a collection of mediators, either by introducing a vector of mediators $W_1,W_2,\ldots$, (and extend Theorem \ref{tm3}) or by applying the method of \cite{li2022bounds} to construct an equivalent mediator $W'$ with an expanded state space that jointly represents the mediator set.

\noindent\textbf{Remark (Scope of PoCs).}
The above theorems are stated in terms of PNS mainly for convenience of presentation. Since PNS is defined through a joint counterfactual event, it represents the most demanding case and makes clear how partial causal information constrains the space of admissible counterfactual distributions. The same set of constraints, however, applies to other probabilities of causation (PoCs). In particular, other PoCs can be obtained from the same optimization program by changing only the objective function. For example, in Theorem~\ref{tm1}, bounds on PN can be derived by replacing the objective with
\[
\max/\min \sum_{j_{k+1},\ldots,j_{n+m+1}} p_{1\ldots k\, j_{k+1}\ldots j_{n+m+1}\,1},
\]
while leaving all constraints unchanged. No additional structural assumptions are required for these extensions.

\section{Simulation Results}
Since \cite{pearl:etal21-r505} already provides toy examples in simple binary settings, this paper focuses on simulated comparing the proposed Theorems~\ref{tm1}, \ref{tm2}, and \ref{tm3} with the Tian--Pearl bounds in Equations~\ref{pnslb} and \ref{pnsub}, as well as the Mueller--Li--Pearl bounds in \cite{pearl:etal21-r505}. Although the proposed bounds strictly generalize the Tian--Pearl and Mueller--Li--Pearl bounds by accommodating multi-valued treatments and outcomes, our comparisons are restricted to binary settings, since both the Tian--Pearl and Mueller--Li--Pearl bounds are only defined for binary variables.

\subsection{Multiple non-descendant covariates with covariate-specific data}

We first consider settings with varying numbers of non-descendant covariates, as illustrated in Figure~\ref{causalg1}. As discussed earlier, the available information consists only of covariate-specific data, namely $P(X,Y,Z_i)$ and $P(Y_X,Z_i)$ for each covariate $Z_i$.

Following \cite{pearl:etal21-r505}, we randomly generate $1000$ sample distributions compatible with the causal diagrams in Figure~\ref{causalg1}, for $n = 1,\ldots,6$, corresponding to Corollary~\ref{tm2} (the observational distributions are generated following the CPT generation in Appendix~E of \cite{pearl:etal21-r505}, and the experimental distributions are computed using the adjustment formula in \cite{pearl1993aspects} based on the generated observational distributions). The Mueller--Li--Pearl bounds are computed using the best single covariate, since no joint covariate distributions are available. Specifically, for each sample distribution, the bounds are computed separately for each covariate, and the tightest lower and upper bounds among them are selected.

For each sampled distribution $i$, let $[a_i,b_i]$ denote the bounds obtained using the proposed theorem, $[c_i,d_i]$ the Tian--Pearl bounds, and $[e_i,f_i]$ the Mueller--Li--Pearl bounds. For each causal diagram, we report the following summary statistics, as shown in Table~\ref{tab:perf}:
\begin{itemize}
    \item Average increase in the Tian--Pearl (TP) lower bound: $\frac{1}{1000}\sum (a_i - c_i)$
    \item Average decrease in the Tian--Pearl upper bound: $\frac{1}{1000}\sum (d_i - b_i)$
    \item Average increase in the Mueller--Li--Pearl (MLP) lower bound: $\frac{1}{1000}\sum (a_i - e_i)$
    \item Average decrease in the Mueller--Li--Pearl upper bound: $\frac{1}{1000}\sum (f_i - b_i)$
    \item Average gap of the Tian--Pearl bounds: $\frac{1}{1000}\sum (d_i - c_i)$
    \item Average gap of the Mueller--Li--Pearl bounds: $\frac{1}{1000}\sum (f_i - e_i)$
    \item Average gap using Theorem~\ref{tm2}: $\frac{1}{1000}\sum (b_i - a_i)$
    \item Number of sample distributions for which Corollary~\ref{tm2} improves upon the Tian--Pearl bounds: $\sum g_i$, where $g_i = 1$ if $(a_i > c_i)$ or $(b_i < d_i)$, and $g_i = 0$ otherwise
    \item Number of sample distributions for which Corollary~\ref{tm2} improves upon the Mueller--Li--Pearl bounds: $\sum h_i$, where $h_i = 1$ if $(a_i > e_i)$ or $(b_i < f_i)$, and $h_i = 0$ otherwise
\end{itemize}

For each \(n = 4, \ldots, 6\), 100 out of 1{,}000 sampled distributions are randomly selected, sorted according to their lower and upper Tian--Pearl bounds, and then used to plot the Tian--Pearl, Müller--Li--Pearl, and proposed bounds (Figures~\ref{res61}--\ref{res66}).

  
  


  
  


\begin{figure}[t]
\centering

\begin{subfigure}[t]{0.45\linewidth}
\centering
\begin{tikzpicture}[->,>=stealth',node distance=1.5cm,
  thick,main node/.style={circle,fill,inner sep=1.5pt}]
  
  \node[main node] (z1) [label=above:{$Z_1$}] {};
  \node (dots) [below=0.2cm of z1] {$\vdots$};
  \node[main node] (zn) [below=0.2cm of dots,label=above:{$Z_n$}] {};
  \node[main node] (x) [below left=1cm of zn,label=left:$X$] {};
  \node[main node] (y) [below right=1cm of zn,label=right:$Y$] {};
  
  \path
    (x) edge (y)
    (z1) edge (x)
    (z1) edge (y)
    (zn) edge (x)
    (zn) edge (y);
\end{tikzpicture}
\caption{$Z_1,\ldots,Z_n$ are non-descendants of $X$.}
\label{causalg1}
\end{subfigure}
\hfill
\begin{subfigure}[t]{0.45\linewidth}
\centering
\begin{tikzpicture}[->,>=stealth',node distance=2cm,
  thick,main node/.style={circle,fill,inner sep=1.5pt}]
  
  \node[main node] (z) [label=above:{$Z$}] {};
  \node[main node] (w) [below=0.8cm of z,label=above:{$W$}] {};
  \node[main node] (x) [below left=1cm of w,label=left:$X$] {};
  \node[main node] (y) [below right=1cm of w,label=right:$Y$] {};
  
  \path
    (x) edge (y)
    (z) edge (x)
    (z) edge (y)
    (x) edge (w)
    (w) edge (y);
\end{tikzpicture}
\caption{$Z$ is the backdoor set and $W$ is the mediator.}
\label{causalg2}
\end{subfigure}

\caption{Causal structures used in the experiments.}
\label{fig:causal_structures}
\end{figure}
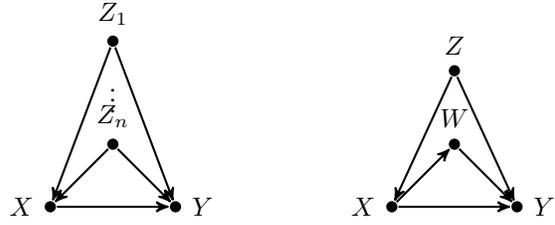

\begin{table*}[t]
\centering
\caption{Performance comparison of bounds with multiple non-descendant covariates using covariate-specific data.}
\label{tab:perf}
\begin{tabular}{l|ccccccccc}
\hline
$n$ 
& $\uparrow$ TP lb 
& $\downarrow$ TP ub 
& $\uparrow$ MLP lb 
& $\downarrow$ MLP ub 
& gap (TP) 
& gap (MLP) 
& gap (Col.~\ref{tm2}) 
& \# imp. TP 
& \# imp. MLP \\
\hline
1 
& 0.0263 & 0.0251 & 0     & 0     & 0.2136 & 0.1622 & 0.1622 & 634 & 0 \\

2 
& 0.0407 & 0.0421 & 0.0007 & 0.0009 & 0.2690 & 0.1878 & 0.1862 & 848 & 565 \\

3 
& 0.0453 & 0.0447 & 0.0020 & 0.0020 & 0.3101 & 0.2242 & 0.2201 & 927 & 673 \\

4 
& 0.0478 & 0.0478 & 0.0033 & 0.0033 & 0.3469 & 0.2578 & 0.2512 & 965 & 727 \\

5 
& 0.0436 & 0.0422 & 0.0037 & 0.0035 & 0.3693 & 0.2907 & 0.2835 & 991 & 787 \\

6 
& 0.0402 & 0.0398 & 0.0043 & 0.0039 & 0.3965 & 0.3247 & 0.3165 & 996 & 818 \\
\hline
\end{tabular}
\end{table*}

\begin{figure}[t]
\centering
\includegraphics[width=0.33\textwidth]{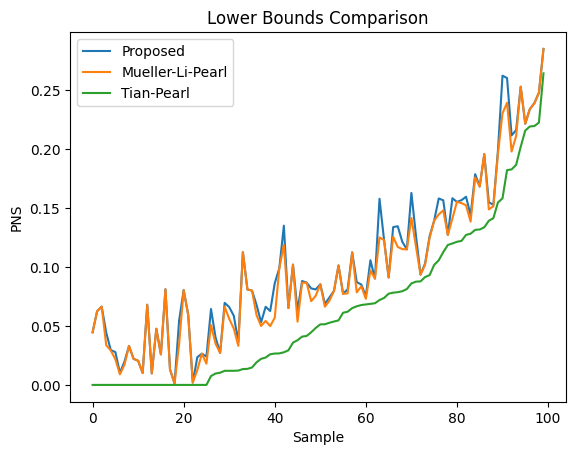}
\caption{PNS lower bounds for $n=6$ in Figure \ref{causalg1}.}
\label{res61}
\end{figure}
\begin{figure}[t]
\centering
\includegraphics[width=0.33\textwidth]{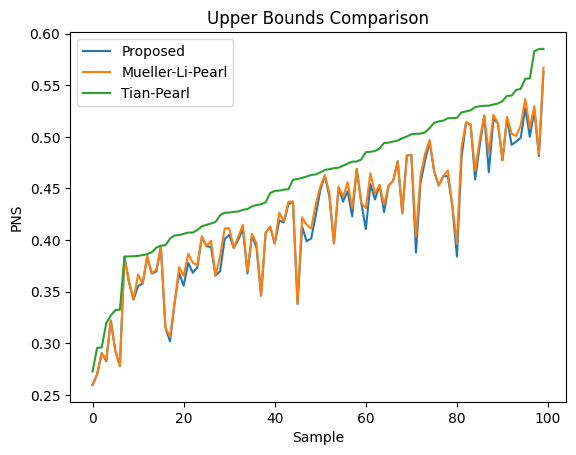}
\caption{PNS lower bounds for $n=6$ in Figure \ref{causalg1}.}
\label{res62}
\end{figure}

\begin{figure}[t]
\centering
\includegraphics[width=0.33\textwidth]{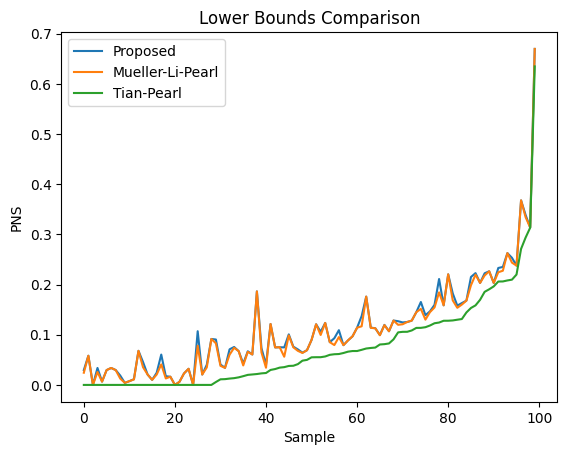}
\caption{PNS lower bounds for $n=5$ in Figure \ref{causalg1}.}
\label{res63}
\end{figure}
\begin{figure}[t]
\centering
\includegraphics[width=0.33\textwidth]{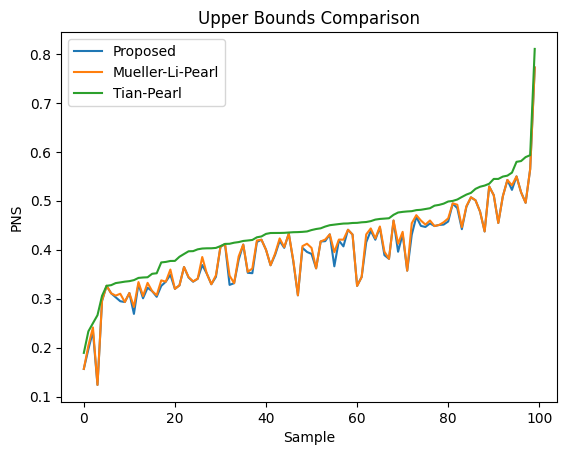}
\caption{PNS lower bounds for $n=5$ in Figure \ref{causalg1}.}
\label{res64}
\end{figure}

\begin{figure}[t]
\centering
\includegraphics[width=0.33\textwidth]{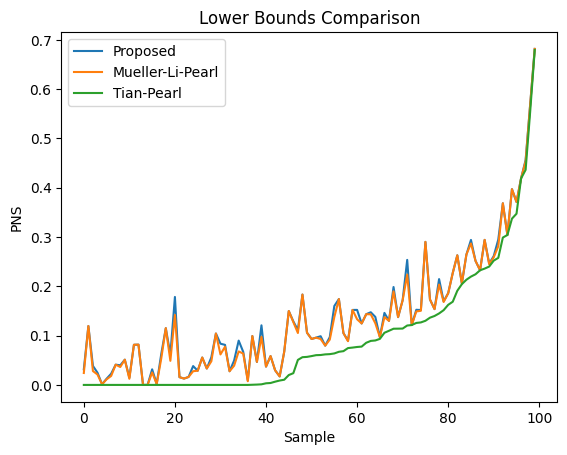}
\caption{PNS lower bounds for $n=4$ in Figure \ref{causalg1}.}
\label{res65}
\end{figure}
\begin{figure}[t]
\centering
\includegraphics[width=0.33\textwidth]{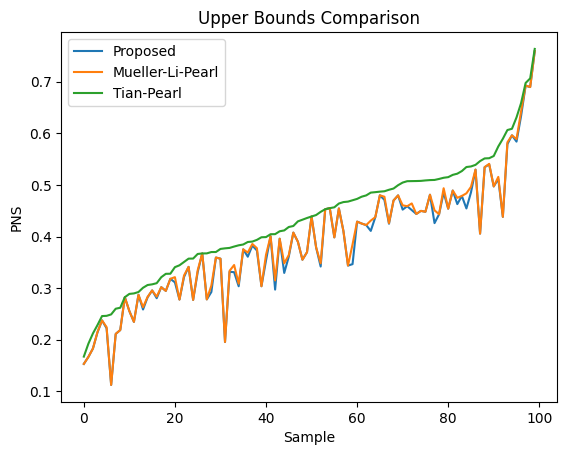}
\caption{PNS lower bounds for $n=4$ in Figure \ref{causalg1}.}
\label{res66}
\end{figure}

Table~\ref{tab:perf} demonstrates that incorporating multiple non-descendant covariates through Corollary~\ref{tm2} consistently tightens the bounds relative to both the Tian--Pearl (TP) and Mueller--Li--Pearl (MLP) baselines. Compared with the TP bounds, the proposed method achieves substantial average improvements in both the lower and upper bounds across all values of $n$, with mean lower-bound increases and upper-bound decreases of approximately $0.02$--$0.04$. These improvements translate into a marked reduction of the bound width: while the average TP gap increases from $0.2136$ to $0.3965$ as $n$ grows, the gap under Corollary~\ref{tm2} remains consistently smaller, increasing only from $0.1622$ to $0.3165$.

Relative to the MLP bounds, Theorem~\ref{tm2} yields modest but steadily growing improvements as the number of covariates increases. Although no improvement is possible when $n=1$, the average gains in both the lower and upper bounds increase monotonically with $n$, and the resulting gaps are uniformly smaller than those of MLP. Correspondingly, the proportion of sample distributions benefiting from Theorem~\ref{tm2} increases rapidly with $n$, exceeding $90\%$ relative to TP for $n \ge 3$, and reaching over $80\%$ relative to MLP by $n=6$. These results confirm that aggregating covariate-specific information across multiple non-descendant covariates can substantially tighten bounds even in the absence of joint covariate distributions.

\subsection{Backdoor Set with Mediators}
As for the mediator case with backdoor set, we compare the proposed Theorem~\ref{tm3} with the Tian--Pearl bounds and the Mueller--Li--Pearl bounds under the causal structure shown in Figure~\ref{causalg2}, which includes one confounder and one mediator. In this setting, the Tian--Pearl bounds do not exploit either additional covariates, the Mueller--Li--Pearl bounds utilize only the confounder, and the proposed theorem incorporates information from both the confounder and the mediator. Following the same simulation protocol as before, we generate $1000$ compatible sample distributions. We report the same summary statistics as before, as shown in Table~\ref{tab:perf_col12}.
\begin{table*}[t]
\centering
\caption{Performance comparison for the mediator case with a back-door set under the causal structure in Figure~\ref{causalg2}.}
\label{tab:perf_col12}
\begin{tabular}{ccccccccc}
\hline
$\uparrow$ TP lb 
& $\downarrow$ TP ub
& $\uparrow$ MLP lb
& $\downarrow$ MLP ub
& gap (TP)
& gap (MLP)
& gap (Thm~\ref{tm3})
& \# impr. TP
& \# impr. MLP \\
\hline
0.0268
& 0.0258
& 0.0057
& 0.0051
& 0.2833
& 0.2417
& 0.2308
& 782
& 245 \\
\hline
\end{tabular}
\end{table*}

The results in Table~\ref{tab:perf_col12} indicate that incorporating mediator information consistently tightens the bounds. Compared to the Tian--Pearl bounds, Theorem~\ref{tm3} achieves a substantial reduction in the average bound gap, highlighting the benefit of exploiting auxiliary structure beyond using only treatment and outcome variables, without considering any covariates. Although the Mueller--Li--Pearl bounds already improve upon Tian--Pearl by conditioning on the confounder, their framework does not allow confounder and mediator information to be incorporated simultaneously. In contrast, the proposed formulation flexibly integrates both sources of information through additional constraints, yielding the smallest average gap among all methods. Moreover, Theorem~\ref{tm3} improves upon the Tian--Pearl bounds for the majority of sampled distributions and also provides nontrivial improvements over the Mueller--Li--Pearl bounds. Due to space limitations, we restrict attention to a single mediator; investigating the effect of multiple mediators is left for future work.

We further randomly select $100$ out of the $1000$ sampled distributions, sort them by their Tian--Pearl lower and upper bounds, and plot the corresponding Tian--Pearl, Mueller--Li--Pearl, and proposed bounds in Figures~\ref{res71} and \ref{res72}.

\begin{figure}[t]
\centering
\includegraphics[width=0.33\textwidth]{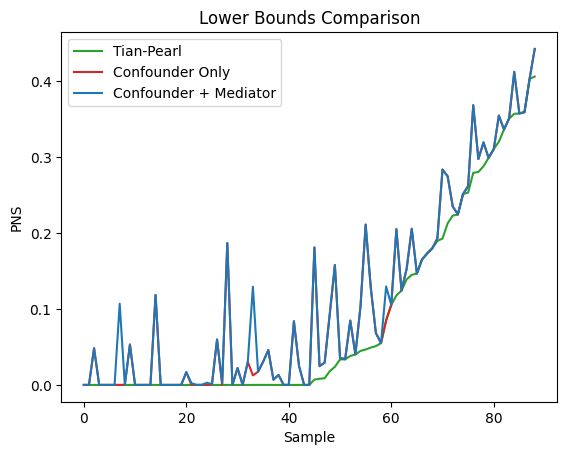}
\caption{PNS lower bounds comparison in Figure\ref{causalg2}.}
\label{res71}
\end{figure}

\begin{figure}[t]
\centering
\includegraphics[width=0.32\textwidth]{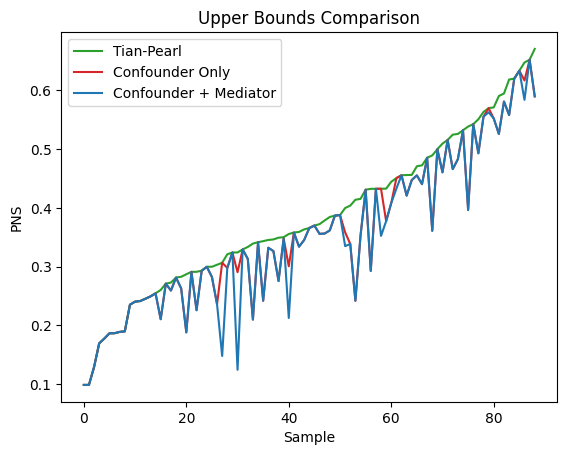}
\caption{PNS upper bounds comparison in Figure \ref{causalg2}.}
\label{res72}
\end{figure}

\section{Conclusion}

This paper studies the problem of bounding probabilities of causation when causal knowledge is incomplete but informative. Rather than assuming either no structural information or a fully specified causal diagram, we propose an optimization-based framework that allows partial causal information to be incorporated as constraints on counterfactual distributions. This formulation provides a unified way to combine experimental data, observational data, and heterogeneous structural assumptions within a single bounding procedure.

Our results extend existing bounds in several directions. First, the framework subsumes the classical Tian--Pearl bounds and graph-based bounds as special cases, while allowing treatments and outcomes to take arbitrary finite values. Second, it enables covariate information to be used in a modular fashion, without requiring joint distributions over all covariates. Third, the same formulation accommodates mediator information alongside back-door covariates, yielding further tightening when such information is available. In each case, the bounds remain formally valid under uncertainty about the remaining causal structure.

Simulation results confirm that the proposed bounds are consistently tighter than existing baselines. In settings with multiple non-descendant covariates, even when only covariate-specific data are available, aggregating information across covariates substantially reduces bound width compared to both Tian--Pearl and Mueller--Li--Pearl bounds. Additional improvements are obtained when mediator information is incorporated. These gains become more pronounced as more partial causal information is available, illustrating the benefit of systematically integrating weak but informative structural assumptions.

Overall, this work demonstrates that meaningful progress on individualized causal questions is possible without full causal identification. By treating partial causal knowledge as constraints in an optimization problem, probabilities of causation can be bounded more tightly in realistic settings where causal structure is incomplete. We view this framework as a step toward making counterfactual reasoning more applicable in domains where causal knowledge is necessarily fragmentary but still informative.

\section{Limitations and future work.}
A limitation of our framework is computational. The programs in Theorems~\ref{tm1} and~\ref{tm2} are linear and can be solved efficiently with standard linear programming solvers, even for moderately sized discrete domains. However, once mediator information is incorporated, Theorem~\ref{tm3} introduces nonlinear constraints (e.g., bilinear equalities enforcing conditional independencies), which turns the optimization into a nonconvex problem and may lead to increased runtime and sensitivity to local optima as the state spaces grow. An important direction for future work is to develop simpler representations of mediator information that preserve most of the tightening effect while recovering tractable programs, for example by constructing lower-dimensional or aggregated mediator summaries, deriving equivalent linear constraints via alternative parameterizations, or designing relaxations that yield efficiently computable certified bounds.

\clearpage
\bibliography{ang.bib}
\bibliographystyle{icml2026}

\newpage
\appendix
\onecolumn
\section{Appendix}
\subsection{Proof of Theorem \ref{tm1}}
\begin{proof}
    The theorem follows directly from the fact that the variables
\( p_{j_1 \ldots j_{n+m+1}} \) represent the joint probability
\( P(Y_{x_1}, \ldots, Y_{x_n}, Z_1, \ldots, Z_m, X) \).
Equation~\ref{thm4eq1} encodes the experimental distribution constraints,
while Equation~\ref{thm4eq2} encodes the observational distribution constraints.

\end{proof}
\subsection{Proof of Theorem \ref{tm2}}
\begin{proof}
    Again, the theorem follows directly from the fact that the variables $p_{j_1 \ldots j_{n+m+1}}$ represent the joint probability $P(Y_{x_1}, \ldots, Y_{x_n}, Z_1, \ldots, Z_m, X)$. Equation~\ref{tm5eq1} encodes the experimental distribution constraints,
while Equation~\ref{tm5eq2} encodes the observational distribution constraints.
\end{proof}
\subsection{Proof of Theorem \ref{tm3}}
\begin{proof}
    The theorem again follows from the fact that the variables $p_{j_1...j_{2n+m+1}}$ represent the joint probability $P(Y_{x_1},...,Y_{x_n},W_{x_1},...,W_{x_n},Z_1,...,Z_m,W,X)$, Equation~\ref{thm6eq1} encodes the experimental distribution constraints, while Equation~\ref{thm6eq2} encodes the observational distribution constraints.
    
    Moreover, since $(Z_1,\ldots,Z_m)$ forms a backdoor set, it follows that
$Y_x \independent X \mid W_x, Z_1,\ldots,Z_m$ for all $x \in {X}$, and
$Y_x \independent W_{x'} \mid W_x, Z_1,\ldots,Z_m$ for all $x,x' \in {X}$ with $x \neq x'$.
These conditional independences can be read off from the twin network representation
\cite{pearl2009causality}. Equation~\ref{thm6eq3} encodes the former independence,
while Equation~\ref{thm6eq4} encodes the latter.

\end{proof}



\end{document}